\documentclass{article}

\usepackage[preprint, nonatbib]{neurips_2022}

\usepackage{times}
\newcommand{\be}{\begin{equation}}
\newcommand{\ee}{\end{equation}}

 \usepackage{amsmath,amsthm,amssymb,amscd}
 \usepackage{amsmath,amssymb,amscd}
 \usepackage{amsfonts}
\usepackage{stmaryrd}
\usepackage{dsfont}
\usepackage{mathrsfs} 
\usepackage{times}
\usepackage{algpseudocode}
\usepackage{comment}
\usepackage{booktabs}
\usepackage{enumerate}
\usepackage[utf8]{inputenc} 
\usepackage[T1]{fontenc}    
\usepackage{hyperref}       
\usepackage{url}            
\usepackage{nicefrac}       
\usepackage{microtype}      
\usepackage{placeins}
\usepackage{algorithm}
\usepackage[square,sort,comma,numbers]{natbib}
\usepackage{multirow}
\usepackage{xcolor}         

\ifdefined \set 
  \renewcommand{\set}[1]{\left\{#1\right\}}
\else
  \newcommand{\set}[1]{\left\{#1\right\}}
\fi

\newtheorem{theorem}{Theorem}
\newtheorem{lemma}{Lemma}

\newtheorem{definition}{Definition}
\newtheorem{example}{Example}
\newtheorem{remark}{Remark}

\def\fp2u{\frac{\partial^2}{\partial u^2}}

\def\noi{\noindent}

\def\bx{{\bf x}}
\def\x{{\bf x}}
\def\bfx{{\bf x}}
\def\RR{{\mathbb R}}
\def\bw{{\bf w}}
\def\w{{\bf w}}
\def\la{\langle}
\def\ra{\rangle}
\def\l{{\ell}}

\def\S{{\cal S}}
\def\lg{\log}

\def\CY{{\cal Y}}
\def\CH{{\cal H}}
\def\hy{\hat{y}}

\renewcommand{\S}{\mathcal{S}}

\usepackage[colorinlistoftodos, textsize=small]{todonotes}

\title{Precise Regret Bounds for Log-loss via a Truncated Bayesian Algorithm}
\author{Changlong Wu\\
CSoI\\
Purdue University\\
\texttt{wuchangl@hawaii.edu}
\and
Mohsen Heidari\\
CSoI\\
Purdue University\\
\texttt{mheidari@purdue.edu}
\and
Ananth Grama \\
CSoI\\
Purdue University\\
\texttt{ayg@cs.purdue.edu}
\and
Wojciech Szpankowski\\
CSoI\\
Purdue University\\
\texttt{szpan@purdue.edu}
}

\begin{document}

\maketitle

\begin{abstract}
We study the sequential general online regression, known also as
the sequential probability assignments, under logarithmic loss
when compared against a broad class of experts.
We focus on obtaining tight, often matching,
lower and upper bounds for the sequential minimax regret that are defined as
the excess loss it incurs over a class of experts.
After proving a general upper bound 
we consider some
specific classes of experts from Lipschitz class to
bounded Hessian class and derive matching lower
and upper bounds with provably optimal constants.
Our bounds work for a wide range of values of the
data dimension and the number of rounds.
To derive lower bounds, we use
tools from  information theory (e.g., Shtarkov sum)
and for upper bounds, we resort to new
``smooth truncated covering'' of the class of experts.
This allows us to find constructive proofs
by applying a simple and novel truncated Bayesian algorithm.
Our proofs are substantially simpler than the existing ones and
yet provide tighter (and often optimal) bounds.
\end{abstract}

\section{Introduction}

In the online learning and/or sequential probability assignments arising in information
theory, portfolio optimization, and machine learning  
the training algorithm consumes $d$ dimensional data in rounds 
$t\in \{1, 2, \ldots, T\}$ and predicts the label $\hat{y}_t$
based on data received and labels observed so far.
Then the true label $y_t$ is revealed and the loss $\ell(y_t,\hat{y}_t)$ is incurred.
The (pointwise) \emph{regret\/} is defined
as the (excess) loss incurred by the algorithm 
over a class of experts, also called the hypothesis class.

More precisely, at each round $t\ge 1$  
the learner obtains a $d$ dimensional input/ feature
vector $\bx_t\in \RR^d$. In addition to $\bx_t$, the learner may use the past 
observations $(\bx_r,y_r)$, $r<t$ to make a prediction $\hy_t$ of the true label. 
Therefore, the prediction can be written as $\hy_t = \phi_t(y^{t-1}, \bx^t)$, 
where $\phi_t$ represents the strategy of the learner to obtain its prediction 
based on the past and current observations. 
Once a prediction is made, the nature reveals the true label $y_t$ and the learner 
incurs loss $\ell: \hat{\mathcal{Y}}\times \mathcal{Y} \rightarrow \RR$, 
where $\hat{\CY}$ and $\CY$ are the prediction and label domains respectively. 
Hereafter, we assume throughout $\hat{\CY}=[0,1]$, $\mathcal{Y}=\{0,1\}$ with logarithmic loss
\begin{equation}
\label{logloss}
    \l(\hy_t, y_t)=-y_t \log (\hy_t)-(1-y_t) \log (1-\hy_t).
\end{equation}

In regret analysis, we are interested in comparing the accumulated loss of the 
learner with that of the best strategy within a predefined class of predictors (experts)  
denoted as $\mathcal{H}$. More precisely, $\CH$  is a collection of predicting 
functions $h:\RR^d \mapsto \hat{\CY}$ with input being $\bx_t$ at time $t$. 
Therefore, given a learner $\phi_t,t>0$ and  $(y_t, \bx_t)_{t=1}^T$
after $T$ rounds the {\it pointwise regret} is defined as
\begin{align*}
R(\hat{y}^T,y^T,\mathcal{H}|\bx^T)=&\sum_{t=1}^T\ell(\hat{y}_t,y_t)-
\inf_{h\in \mathcal{H}}\sum_{t=1}^T\ell(h(x_t),y_t),
\end{align*}
where $\hat{y}_t=\phi_t(y^{t-1},\bx^{t})$. Observe that the first 
term above represents the accumulated loss incurred by the learning algorithm,
while the second summation deals with the best prediction within $\mathcal{H}$.  
There are two useful perspectives on analyzing the regret,  highlighted next.

\noindent\textbf{Fixed Design:} This point of view studies the minimal regret 
for the worst realization of the label with the feature vector $\bx^T$ known in advance. 
Suppose that the learner has a fixed strategy $\phi_t, t>0$. Then, the 
{\it fixed design minimax regret} for a given $\bx^T$ is defined as
\begin{equation}
\label{eq-regret-fd1}
r_T^*(\CH|\bx^T)= \inf_{\phi^T} \sup_{y^T}R(\phi^T,y^T,\mathcal{H}| \bx^T). 
\end{equation}
Further, the fixed design \textit{maximal}
minimax regret is 
\begin{equation}
\label{eq-mx}
r^*_T(\CH)=\sup_{\bx^T} \inf_{\phi^T} \sup_{y^T}R(\phi^T,y^T,\mathcal{H}| \bx^T).
\end{equation}

\noindent\textbf{Sequential Design:} 
In this paper  we mostly focus on the {\it sequential} or {\it agnostic} regret
in which the optimization on regret is  performed at every time $t$ 
without knowing  in advance $\bx^T$ or $y^T$. Then the 
\textit{sequential (maximal) minimax regret} is
\cite{rakhlin2010online}
\begin{equation}
\label{eq-regret-fxm}
r^a_T(\mathcal{H})=\sup_{\bx_1}\inf_{\hat{y}_1}\sup_{y_1}\cdots 
\sup_{\bx_T}\inf_{\hat{y}_T}\sup_{y_T}
R(\hat{y}^T, {y}^T, \CH| \bx^T).
\end{equation}
In \cite{wu-isit22} it is shown that
$r^a_T(\mathcal{H}) \ge r^*_T(\mathcal{H})$ for all $\CH$.
We will use $r^*_T(\CH)$ as our tool to derive lower bounds 
for $r^a_T(\mathcal{H})$.

Our main goal is to gain insights into the growth of the sequential regret $r^a_T(\mathcal{H})$
for various classes $\CH$ and to show how the structure of $\CH$, 
as well as the relationship between
$d$ and $T$ impact the precise growth of the regret. To see this more clearly,
we briefly review the regret in universal source coding.

\paragraph{Regrets in Information Theory.}
In universal compression the dependence between
the regret and the reference class was intensively studied 
\cite{ds04,os04,rissanen96,shamir06b,spa98,xb97,xb00}.
Here there is no feature
vector $\bx^t$ and the dimension $d=1$. 
A sequence $y^T$ is generated by a source $P$ that belongs to a class of sources $\S$, which can be viewed as the reference class $\CH$ in 
online learning. The minimax regret for 
the logarithmic loss is then \cite{davisson73,shtarkov87,ds04} 
$
r_T^*(\S)= \min_{Q} \max_{y^T}
[ -\log Q(y^T) +\lg \sup_{P\in \S} P(y^T)], 
$
where $Q$ is the universal probability assignment approximating the unknown $P$.
The main question is how the structure of $\S$ impacts the growth of the minimax regret.
Let $m$ denote the alphabet size (in online learning, we only consider $m=2$).
It is known \cite{ds04,os04,rissanen96,shamir06b,spa98,xb97,xb00}
that for Markov sources of order $r$ the regret grows like
$m^r(m-1)/2 \log T$ for fixed $m$
\cite{rissanen96,os04,shamir06b,sw12} while in \cite{sw12} 
the minimax regret was analyzed for
all ranges of $m$ and $T$.
For non-Markovian sources the growth is super logarithmic. For example,
for renewal sources of order $r$ the regret is $\Theta(T^{r/(r+1)})$
\cite{cs95} and the precise constant in front of the leading term is know
for $r=1$ \cite{fs02}. 
However, it should be pointed out that \cite{cb94,bry98} studied the general 
classes of densities smoothly parameterized
by a $d$-dimensional data to obtain general results for the minimax regret that can be phrased as an online regret.

\paragraph{Main Contributions.}
Our main results are summarized in Table~\ref{tab:table1}. 
One of the main contributions of this paper is the concept of
a global sequential covering used to prove \emph{constructively} general upper bounds on the regret  (Theorem \ref{th1}). We establish Theorem \ref{th1} via a novel smooth truncation approach enabling us to find tight upper bounds that subsume the state-of-the-art results (e.g., \cite{rakhlin2015sequential,bilodeau2020tight}) obtained non-algorithmically. In fact, Algorithm \ref{alg:2} developed in this paper achieves these bounds. 
Moreover, Theorem~\ref{th1}
 provides optimal constants that are crucial to derive the best bounds in special cases
discussed next.
For general Lipschitz parametric class $\CH$, in Theorem~\ref{th2}, we derive
the upper bound $d\log(T/d)+O(d)$ for $T>d$. In Theorem~\ref{th3}, we
show that the leading constant $1$ (in front of  $d\log(T/d)$) is optimal for $T\gg d\log(T)$.
Furthermore, we obtain 
the best constant for the leading term $\frac{d}{2}\log (T/d)$ when the Hessian of $\log f$  is bounded for any function $f\in \CH$ (see Theorem \ref{th4}). Then, we show in Theorem~\ref{th5} that the constant $\frac{1}{2}$ in our bound is optimal for functions
of the form $f(\langle\w,\x\rangle)$, where $\w\in \RR^d$ is the parameter of the function, $\langle\w,\x\rangle$ is the inner product in $\RR^d$,  $\w$ and $\x$ are in a general $\ell_s$-norm unite ball, and $T \gg
d^{(s+2)/s}$. This result recovers all the lower bounds
in~\cite{shamir2020logistic} obtained  for the logistic regression (however, the technique of \cite{shamir2020logistic} works for other
functions with bounded second derivatives, like the probit function).
Lastly, when $d\ge T$ , for a linear function
of the form $|\langle\w,\x\rangle|$ we show that the growth is at least
$\Omega(T^{s/(s+1)})$ under  $\ell_s$ ball and at most $\tilde{O}(T^{2/3})$
for  $\ell_2$ ball (see Theorem~\ref{th7} and Example~\ref{ex2}).

The main technique used in our paper (smooth truncation) is novel with potential other applications
(e.g., average minimax regret). Instead of the conventional approach for truncating only the values close to $\{0,1\}$, we truncate all values in $[0,1]$ in a smooth way (see Lemma~\ref{lem3}). This
allows us to obtain an upper bound 
via a simple truncated Bayesian algorithm. Our
proofs are substantially simpler (cf.~\cite{bilodeau2020tight})
yet provide tighter and often optimal bounds.

In summary, our main contributions are: (i) constructive proofs through
a new smooth truncated Bayesian algorithm, (ii) the novel concept of global sequential covering,
(iii) lower and upper bounds with optimal leading constants, and at last (iv) novel
information-theoretic technique for the lower bounds.


\setlength{\tabcolsep}{0.5em} 
\renewcommand{\arraystretch}{1.5}

\small

\begin{table}[h!]
  \begin{center}
    \caption{Summary of results}
    \label{tab:table1}
    \begin{tabular}{l|c|c|r}
    \toprule
      \textbf{Constrains} & $d\text{ v.s }T$ & \textbf{Bounds} & \textbf{Comment}\\
      \midrule
      General $\alpha$ cover $\mathcal{G}_{\alpha}$ & N.A. & $\displaystyle 
r^a_T(\mathcal{H})\le \inf_{0<\alpha<1}\left\{2\alpha T+\log 
|\mathcal{G}_{\alpha}|\right\}$& Theorem~\ref{th1} \\
      \hline
      General Lipschitz $f$  & Any & 
$r^a_T(\mathcal{H}_f)\le d\log\left(\frac{T}{d}+1\right)+O(d)$ &Theorem~\ref{th2}\\
      under $\ell_s$ ball & $T\gg d\log T$ & 
$r^a_T(\mathcal{H}_f)\ge d\log\left(\frac{T}{d}\right)-O(d\log\log T)$ &Theorem~\ref{th3}\\
      \hline
      Bounded Hessian &\multirow{2}{*}{Any} 
&\multirow{2}{*}{ $\displaystyle 
r^a_T(\mathcal{H}_f)\le\frac{d}{2}\log\left(\frac{T}{d}+1\right)+O(d)$} & 
\multirow{2}{*}{Theorem~\ref{th4}}\\
      of $\log f$ under $\ell_2$ ball &&&\\\hline
      $f(\langle \w,\x\rangle)$ with $f'(0)\not=0$ & \multirow{2}{*}{$T\gg d^{(s+2)/s}$}& 
\multirow{2}{*}{$r^a_T(\mathcal{H}_f)\ge \frac{d}{2}\log\left(\frac{T}{d^{(s+2)/s}}\right)-O(d)$}
&\multirow{2}{*}{Theorem~\ref{th5}}\\
      under $\ell_s$ ball &  & 
& \\\hline
      $|\langle\w,\x\rangle|$ under $\ell_s$ ball & 
$d\ge T$ & $r^a_T(\mathcal{H}_f)\ge \frac{s+1}{s\cdot e}T^{s/(s+1)}$ 
&Theorem~\ref{th7}\\
      $|\langle\w, \x\rangle|$ under $\ell_2$ ball & $d\ge T$ & $\Omega(T^{2/3})\le r^a_T(\mathcal{H}_f)\le \tilde{O}(T^{2/3})$  & Example~\ref{ex2}\\\bottomrule
    \end{tabular}
  \end{center}
\end{table}

\normalsize

\paragraph{Related Work}
In this paper we study the sequential minimax regret for a
general online regression with logarithmic loss using tools of
information theory, in particular the universal source coding (lower bounds)
\cite{bry98,ds04,kt83,os04,rissanen84,rissanen96,shamir06b,xb97}
and sequential covering 
(upper bounds).

Most of the existing works in online
regression deal with the logistic regression.
We first mention the work of \cite{hazan14} who studied the
pointwise regret of the logistic regression for the
{\it proper} setting.
Unlike the {\it improper} learning, studied in this paper, where
feature $\bx_t$ at time $t$ is also available to the learner,
\cite{hazan14} showed that the pointwise regret is
$\Theta(T^{1/3})$ for $d=1$ and $O(\sqrt{T})$ for $d>1$.
Furthermore, \cite{kn05} demonstrates results
that regret for logistic regression
grows like $O(d \log T/d)$.
This was further generalized in \cite{foster18}.
These results were strengthened in \cite{shamir2020logistic}, which also provides
the matching lower bounds. Precise asymptotics for the fixed design minimax
regret were recently presented in \cite{jss20,jss21} for $d=o(T^{1/3})$.

Regret bounds under logarithmic loss for general expert class $\mathcal{H}$
was first investigated by Vovk under the framework of mixable losses
\cite{jezequel2021mixability,vovk2001competitive}.  In
particular, Vovk showed that for finite class $\mathcal{H}$, the regret
growth is $\log |\mathcal{H}|$ via the \emph{aggregating algorithm} (i.e.,
the Bayesian algorithm that we will discuss below). We refer the reader
to~\cite[Chapter 3.5, 3.6]{lugosi-book} and the references therein for more
results on this topic. Cesa-Bianchi and Lugosi \cite{lugosi-book} were the first to
investigate log-loss under general (infinite) expert class
$\mathcal{H}$~\cite[Chapter 9.10, 9.11]{lugosi-book}, where they derived a
general upper bound using the concept of covering number and a two-stage
prediction scheme. In particular, Cesa-Bianchi and Lugosi showed that for
Lipschitz parametric classes with values bounded away from $\{0,1\}$, one
can achieve a regret bound of the form $d/2\log(T/d)$. When the values are
close to $\{0,1\}$, they used a \emph{hard} truncation approach, which
gives a sub-optimal bound of the form $2d\log(T/d)$. Moreover, the approach
of \cite{lugosi-book} only works for the fixed design regret (or
\emph{simulatable} in their context). In~\cite{rakhlin2015sequential}, the
authors extended the result of~\cite[Chapter 9.10]{lugosi-book} to the
sequential case via the machinery of sequential covering that was
established in~\cite{rakhlin14}. However,~\cite{rakhlin2015sequential} also
used the same \emph{hard} truncation as in~\cite{lugosi-book} resulting
in suboptimal upper bounds. In~\cite{bilodeau2020tight}, the authors
obtained an upper bound similar to the upper bound presented in Theorem~\ref{th1}
via the observation that the $\log$ function is self-concordance. In
particular, this allows them to resolve the tight  bounds for
non-parametric Lipschitz functions map $[0,1]^s\rightarrow [0,1]$.


\section{Problem Formulation and Preliminaries}

We denote $\mathcal{X}$ as the input feature space and $\mathcal{H}$ as the concept class which is a set of functions mapping $\mathcal{X}\rightarrow [0,1]$. We often use an auxiliary set $\mathcal{W}$ to index $\mathcal{H}$. We say a function $g$ is \emph{sequential} if it maps $\mathcal{X}^*\rightarrow [0,1]$, where $\mathcal{X}^*$ is set of all finite sequences with elements in $\mathcal{X}$. We denote $\mathcal{G}$ as a class of \emph{sequential} functions. If $T$ is a time horizon, then for any $t\in [T]$, we write $\x^t=\{\x_1,\cdots,\x_t\}$ and $y^t=\{y_1,\cdots,y_t\}$. We use standard asymptotic notation $f(t)=O(g(t))$ or $f(t) \ll g(t)$ if there exist constant $C$ such that $f(t)\le Cg(t)$ for all $t\ge 0$.

The main objective of this paper is to study the growth of the
sequential minimax regret $r^a_T(\CH)$ for a large class of experts $\CH$.
We accomplish it with two different techniques.
For the lower bound, we precisely estimate the  fixed design minimax regret $r_T^*(\CH|\bx^T)$
using the Shtarkov sum \cite{shtarkov87}, discussed next. For the upper bound, we construct a novel global
cover set ${\mathcal G}$ of $\CH$ and design a new (truncated) Bayesian algorithm 
to find precise bounds with constants that are provably optimal. 

\paragraph{Lower Bounds.} We investigate the lower bound of the adversarial regret
$r^a_T(\CH)$ by considering its corresponding fixed design minimax regret $r^*_T(\CH|\bx^T)$ and
$r^*_T(\CH)=\max_{\bx^T} r^*_T(\CH|\bx^T)$.
We are able to do it due to the recent result \cite{wu-isit22} which we quote next.

\begin{lemma}[Wu et al., 2022]
\label{th-general}
Let $\CH$ be any general hypothesis class and $\l$ be any loss function.
Then $r^a_T(\mathcal{H}|\bx^T)=r^*_T(\mathcal{H}|\bx^T)$
for any $\bfx^T \in \mathcal{X}^T$.
Furthermore, 
$$r^a_T(\mathcal{H})\ge r^*_T(\mathcal{H}),$$
and the inequality is strict for certain $\mathcal{H}$, and loss function $\ell$.
\end{lemma}

We establish precise growth of $r^*_T(\CH)$ 
by estimating the Shtarkov sum 
that was intensively analyzed in information theory \cite{shtarkov87,ds04}
and recently applied in online learning \cite{ss20, jss20}.
For the logarithmic loss,
the Shtarkov sum (conditioned on $\x^T$) is defined as follows \footnote{Note that the Starkov sum can be defined for any class of measures, however, here we only use the form for product measures.}
$$S_T(\mathcal{H}|\bx^T)\overset{\text{def}}{=}\sum_{y^T\in \{0,1\}^T}\sup_{h\in\mathcal{H}}P_h(y^T\mid \x^T),$$
where $P_h(y^T\mid \x^T)=\prod_{t=1}^Th(\x_t)^{y_t}(1-h(\x_t))^{1-y_t}$ and we \emph{interpret}
$h(\bx_t) =P(y_t=1| \bx_t)$. The regret can be expressed in terms of the Shtarkov sum (see~\cite[Equation (6)]{jss20} or~\cite[Theorem 9.1]{lugosi-book}) as 
\begin{align}
\label{eq-shtarkov}
    r^*_T(\mathcal{H})=\sup_{\bx^T}\log S_T(\mathcal{H}|\bx^T).
\end{align}
 It is known that the leading term in the Shtarkov sum for a large class of 
$\CH$ is often independent of $\bx^T$ \cite{shamir2020logistic,foster18,jss20,jss21}.
Therefore,  the Shtarkov sum often gives the leading growth of $r_T^*(\CH|\bx^T)$ independent of $\bx^T$, which also suggests the leading growth of the agnostic regret $r^a_T(\CH)$.

\paragraph{Upper Bounds.}
We now discuss our constructive approach to upper bounds. In the next section,
we present our Smooth truncated Bayesian Algorithm (Algorithm~\ref{alg:2}) that provides a constructive
and often achievable upper bound. 
Here we focus on some, mostly known, preliminaries.

Let $\mathcal{G}$ be any reference class map $\mathcal{X}^*\rightarrow [0,1]$. 
Let $\mathcal{W}$ be an index set of $\mathcal{G}$ and $\mu$ be an 
arbitrary finite measure over $\mathcal{W}$. The standard Bayesian predictor with prior 
$\mu$ is presented in Algorithm~\ref{alg:1}. Based on this algorithm, we have the following
two lemmas~\cite[Chapter 3.3]{lugosi-book} that are used to establish most of the upper bounds in this paper. For completeness, we provide simple proofs in Appendix~\ref{app:1}.

\begin{algorithm}[h]
\caption{Bayesian predictor}\label{alg:1}
\textbf{Input}: Reference class $\mathcal{G}:=\{g_w: w\in \mathcal{W}\}$ with index set 
$\mathcal{W}$ and prior $\mu$ over $\mathcal{W}$

\begin{algorithmic}[1]
\State Set $p_w(y^0\mid \bx^0)=1$ for all $w\in \mathcal{W}$.
\For{$t=1,\cdots, T$}
\State Receive feature vector $\bx_t$
\State Make prediction with the following equation: 
    $$\hat{y}_t=\frac{\int_{\mathcal{W}} g_w(\bx^t)p_w(y^{t-1}\mid \bx^{t-1})
\text{d}\mu}{\int_{\mathcal{W}} p_w(y^{t-1}\mid \bx^{t- 1})\text{d}\mu}.$$
\State Receive label $y_t$
\State For all $w\in \mathcal{W}$, update:
    $~p_w(y^t\mid \bx^t)=e^{-\ell(g_w(\bx^t),y_t)}p_w(y^{t-1}\mid \bx^{t-1})$.
\EndFor
\end{algorithmic}



\end{algorithm}

\begin{lemma}
\label{lem1}
Let $\mathcal{G}$ be a collection of functions $g_w: \mathcal{X}^*\rightarrow [0,1], w\in \mathcal{W}$. Let $\hat{y}_t$ be the  Bayesian 
prediction rule as in Step 4 of  Algorithm~\ref{alg:1} with prior $\mu$. Then, for any $\x^T$ and $y^T$ we have
$$
\sum_{t=1}^T\ell(\hat{y}_t,y_t)\le -\log \frac{\int_{\mathcal{W}} p_w(y^T\mid \x^T)\mathrm{d}\mu}{\int_{\mathcal{W}} 1 
\mathrm{d}\mu},
$$
where $p_w(y^T\mid \x^T)=e^{-\sum_{t=1}^T\ell(g_w(\x^t),y_t)}$ and $\ell$ is the log-loss as in equation~(\ref{logloss}).
\end{lemma}

The following lemma bounds the regret under log-loss of finite classes, 
which is well known. 

\begin{lemma}
\label{lem2}
For any finite class of experts $\mathcal{G}$, we have
$r^a_T(\mathcal{G})\le \log |\mathcal{G}|.$
\end{lemma}

\section{Main Results}
\label{sec-main}


We start with a novel covering set called the  {\it global sequential cover} that
is different than the one used in \cite{bilodeau2020tight}: 
\begin{definition}[Global sequential covering]
\label{def1}
For any $\mathcal{H}\subset [0,1]^{\mathcal{X}}$, we say class $\mathcal{G}$ of functions map $\mathcal{X}^*\rightarrow [0,1]$ is a global 
\emph{sequential} $\alpha$-covering of $\mathcal{H}$ at scale $T$ if for any 
$\x^T\in \mathcal{X}^T$ and $h\in \mathcal{H}$, there
exists $g\in \mathcal{G}$ such that $\forall t\in [T]$, 
$$|h(\x_t)-g(\x^t)|\le \alpha.$$
Throughout we assume that $0 < \alpha <1$.
\end{definition}

Note that the main difference between  our \emph{global} sequential covering and  the sequential covering used in~\cite{bilodeau2020tight} is that our covering function \emph{does not} depend on the underlying trees introduced  in~\cite{rakhlin2010online} \footnote{Note that the covering functions in Definition~\ref{def1} can be viewed as the experts constructed in~\cite[Section 6.1]{rakhlin2010online}}. This is crucial to apply our covering set directly in an algorithmic way (see Algorithm~\ref{alg:2}). 
Particularly, it enables us to establish our lower and upper bounds for Lipschitz classes of
functions with the optimal constants on the leading term. We further improve these results for
Lipschitz class with bounded Hessian. Finally, we study the cases when the data dimension $d$ is growing faster than $T$ by bounding the covering size through the sequential fat-shattering number. In particular, we prove matching (up to $\text{poly}\log T$ factor) upper and lower bounds for the generalized linear functions.
\paragraph{General Results.}
We are now in the position to state our first main general finding.

\begin{theorem}
\label{th1}
If for any $\alpha>0$ there exists a 
global sequential $\alpha$-covering set $\mathcal{G}_{\alpha}$ of $\mathcal{H}$, then
\begin{equation}
\label{eq-general}
r^a_T(\mathcal{H})\le\inf_{0< \alpha< 1}\left\{2\alpha T+\log|\mathcal{G}_{\alpha}|\right\},
\end{equation}
and this bound is achievable algorithmically.
\end{theorem}

We should point out that Theorem~\ref{th1} also improves the results of 
\cite{bilodeau2020tight} by obtaining better constants in front of both $\alpha T$
and $\log |\mathcal{G_\alpha}|$. The proof is based on the following key lemma that is established in  Appendix~\ref{app:2}.

\begin{algorithm}[ht]
\caption{Smooth truncated Bayesian predictor}\label{alg:2}
\textbf{Input}: Reference class $\mathcal{G}$ with index set $\mathcal{W}$ and prior $\mu$ over $\mathcal{W}$, and truncation parameter $\alpha$

\begin{algorithmic}[1]
\State Let $p_w(y^0\mid \x^0)=1$ for all $w\in \mathcal{W}$
\For{$t=1,\cdots, T$}
\State Receive feature $\x_t$
\State For all $w\in \mathcal{W}$, set
                $$\tilde{g}_w(\x^t)=\frac{g_w(\x^t)+\alpha}{1+2\alpha}$$
\State Make prediction
    $$\hat{y}_t=\frac{\int_{\mathcal{W}} \tilde{g}_w(\x^t)p_w(y^{t-1}\mid \x^{t-1})
\text{d}\mu}{\int_{\mathcal{W}} p_w(y^{t- 1}\mid \x^{t-1})\text{d}\mu}$$
\State Receive label $y_t$
\State For all $w\in \mathcal{W}$, update:
    $~p_w(y^t\mid \x^t)=e^{-\ell(\tilde{g}_w(\x^t),y_t)}p_w(y^{t-1}\mid \x^{t-1})$.
\EndFor
\end{algorithmic}



\end{algorithm}

\begin{lemma}
\label{lem3}
Suppose $\mathcal{H}$ has a global sequential $\alpha$-covering set $\mathcal{G}$ for some $\alpha\in [0,1]$.
Then, there exists a truncated set $\tilde{\mathcal{G}}$ of $\mathcal{G}$ with $|\tilde{\mathcal{G}}|= |\mathcal{G}|$ 
such that for all $\x^T,y^T$ and $h\in \mathcal{H}$ there exists a 
$\tilde{g}\in \tilde{\mathcal{G}}$ satisfying
\begin{equation}
\label{eq1-lema}
\frac{p_h(y^T\mid \x^T)}{p_{\tilde{g}}(y^T\mid\x^T)}\le\left(1+2\alpha\right)^T,
\end{equation}
where
$$
p_h(y^T\mid \x^T)=\prod_{t=1}^T h(\x_t)^{y_t}(1-h(\x_t))^{1-y_t} \ \ \ 
{\rm and} \ \ \ 
p_{\tilde{g}}(y^T\mid \x^T)=\prod_{t=1}^T\tilde{g}(\x^t)^{y_t}(1-\tilde{g}(\x^t))^{1-y_t}.
$$
\end{lemma}
\begin{proof}[Proof of Theorem 1]
We show that for any $0<\alpha<1$ if an $\alpha$-covering set 
$\mathcal{G}_{\alpha}$ exists, then one can achieve the claimed bound 
for such an $\alpha$. To do so, we run the Smooth truncated Bayesian Algorithm (Algorithm~\ref{alg:2})
on $\mathcal{G}_{\alpha}$ with uniform prior and truncation parameter $\alpha$. We denote by
$\tilde{\mathcal{G}}_{\alpha}$ to be
the truncated class of $\mathcal{G}_{\alpha}$ as in Lemma~\ref{lem3} (same as the step $4$ of Algorithm~\ref{alg:2}).
We now fix $\x^T,y^T$.  By Lemma~\ref{lem2} (with $\mathcal{G}$ being $\tilde{\mathcal{G}}_{\alpha}$), we have
$$
\sum_{t=1}^T\ell(\hat{y}_t,y_t)\le \inf_{\tilde{g}\in \tilde{\mathcal{G}}_{\alpha}}
\sum_{t=1}^T\ell(\tilde{g}(\x^t),y_t)+\log|\tilde{\mathcal{G}}_{\alpha}|=
\inf_{\tilde{g}\in \tilde{\mathcal{G}}_{\alpha}}\sum_{t=1}^T\ell(\tilde{g}(\x^t),y_t)
+\log|\mathcal{G}_{\alpha}|,
$$
the last equality follows from  $|\mathcal{G}_{\alpha}|=|\tilde{\mathcal{G}}_{\alpha}|$. 
Since $\sum_{t=1}^T\ell(f(\x^t),y_t)=-\log p_f(y^T\mid \x^T)$ for any function $f$, then by Lemma~\ref{lem3} we conclude that
$$
 \inf_{h\in \mathcal{H}}\sum_{t=1}^T\ell(h(\x_t),y_t)\ge 
 \inf_{\tilde{g}\in\tilde{\mathcal{G}}_{\alpha}}\sum_{t=1}^T\ell(\tilde{g}(\x^t),y_t)-
 T\log\left(1+2\alpha\right).
 $$
The result follows by combining the inequalities and noticing that 
$\log(1+x)\le x$ for all $x\ge -1$.
\end{proof}

We further note that for any constant $c_1,c_2$ for which the bound 
$r^a_T(\mathcal{H})\le c_1\alpha T+c_2\log |\mathcal{G}_{\alpha}|$ holds universally we must have $c_1\ge 2$ and $c_2\ge 1$. Therefore, our bounds are optimal on
the constants\footnote{Note that Theorem~\ref{th-general} can be easily generalized to alphabet of size $K$ by using truncation $\tilde{g}(y_t\mid \x^t)=\frac{g(y_t\mid \x^t)+\alpha}{1+K\alpha}$. The derived bound will be of form $K\alpha T+\log|\mathcal{G}_{\alpha}|$ . This is tight on the constant $K$ as well.}. To see this, we let 
$\mathcal{X}=[T]$ and define $g$ to be the function that maps every 
$t\in [T]$ to $\frac{1}{2}$. Let $\mathcal{H}$ be the class of functions 
that maps to $[1/2-\alpha,1/2+\alpha]$. Clearly, $\mathcal{H}$ is $\alpha$-covered by $g$.
By noting that the maximum probability 
is $(1/2+\alpha)^T = (1+2\alpha)^T (1/2)^T$, we compute the
Shtarkov sum (\ref{eq-shtarkov}) to get:
$$
r^a_T(\mathcal{H})\ge r^*_T(\mathcal{H})\ge \log (1+2\alpha)^T\sim 2\alpha T,
$$
where $\sim$ holds when $\alpha$ is sufficiently small. 
This implies that we must have $c_1\ge 2 $. The fact that $c_2\ge 1$ is due to that the mixability constant of log-loss is $1$, which also follows from Theorem~\ref{th3} below.

\paragraph{Lipschitz Parametric Class.}
We now consider a Lipschitz parametric function class. Given a function 
$f:\mathcal{W}\times \mathcal{X}\rightarrow [0,1]$, define the following class
$$
\mathcal{H}_f=\{f(\w,\cdot)\in [0,1]^{\mathcal{X}} ~:~ \w\in \mathcal{W}\},
$$
where $\bw \in \mathcal{W}$ is often a $d$-dimensional vector in $\RR^d$.

We will assume that $f(\w,\x)$ is $L$-Lipschitz on $\w$ for every $\x$, where $L\in \mathbb{R}^{+}$. 
More formally, $\forall \w_1,\w_2\in \mathcal{W}$ and $\x\in \mathcal{X}$ we have
$$
|f(\w_1,\x)-f(\w_2,\x)|\le L||\w_1-\w_2||,
$$
where $||\cdot||$ is some norm on $\mathcal{W}$. For example, 
if we take $\mathcal{W}\subset \mathbb{R}^d$ then the norm can be 
$\ell_1$, $\ell_2$ or $\ell_{\infty}$ norm. For any specific norm 
$||\cdot||$, we write $\mathcal{B}(R)$ for the ball under 
such norm with radius $R$ in $\mathcal{W}$. In particular, 
we denote by $\mathcal{B}_s^d(R)$ the ball in $\mathbb{R}^d$ of radius 
$R$ under $\ell_{s}$ norm centered at the origin.

\begin{theorem}
\label{th2}
Let $f:\mathcal{B}_s^d(R)\times \mathbb{R}^d\rightarrow [0,1]$ 
be a $L$-Lipschitz function under $\ell_s$ norm. Then
\begin{equation}
\label{eq-th2}
r^a_T(\mathcal{H}_f)\le \min\left\{ d\log\left(\frac{2RLT}{d}+1\right)+2d, T\right\}.
\end{equation}
\end{theorem}
\begin{proof}
By $L$-Lipschitz condition, to find an $\alpha$-covering in the sense of Definition~\ref{def1}, 
we only need to find a covering of $\mathcal{B}_s^d(R)$ with radius 
$\alpha/L$. By standard result (see e.g. Lemma 5.7 and Example 
5.8 of~\cite{wainwright2019high}) we know that the covering size is upper bounded by
$$\left(\frac{2RL}{\alpha}+1\right)^d.$$
By Theorem~\ref{th1}, we find
$$
r^a_T(\mathcal{H}_f)\le\inf_{0<\alpha<1}\left\{2\alpha T+d\log
\left(\frac{2RL}{\alpha}+1\right)\right\}.
$$
Taking $\alpha=d/T$, we conclude
$$
r^a_T(\mathcal{H}_f)\le d\log\left(\frac{2RLT}{d}+1\right)+2d.
$$
This completes the proof for $T\ge d$. The upper bound $T$ is achieved by predicting $\frac{1}{2}$ every time.
\end{proof}

\begin{example}
{\rm 
For logistic function $f(\w,\x)=(1+e^{-\langle\w,\x\rangle})^{-1},$ 
and  $\w\in\mathcal{B}_2^d(R)$ with  $\x\in \mathcal{B}_2^d(1)$  our result recovers those of~\cite{foster18}, 
but with a better leading constant (the bound in~\cite{foster18} 
has a constant $5$). Note that, the result 
in~\cite{bilodeau2020tight} also provides a sub-optimal constant $c\sim 4$. Moreover, our bounds have a logarithmic dependency on Lipschitz constant $L$.
}
\end{example}

The question arises whether the factor in front of $\log T$ can be improved 
to $d/2$ instead of $d$ as discussed  in some recent papers
\cite{shamir2020logistic,jss20,jss21}. 
In Theorem~\ref{th3} below, we show that, in general, it cannot
unless we further strengthen our assumption (see Theorem~\ref{th4}).
For the ease of presentation, we only consider the parameters restricted to $\ell_2$ norm.
The proof can be found in Appendix~\ref{app:3}.

\begin{theorem}
\label{th3}
For any $d,T,R, L$ such that $T\gg d\log (RLT)$, there exists $L$-Lipschitz
function $f:\mathcal{B}_2^d(R)\times \mathbb{R}^d\rightarrow [0,1]$ such that
\begin{equation}
\label{eq-th3}
r^a_T(\mathcal{H}_f)\ge d\log \left(\frac{RLT}{d}\right) -d\log 64- d\log\log (RLT).
\end{equation}
\end{theorem}

\paragraph{Lipschitz Class with Bounded Hessian.}
As we have demonstrated in Theorem~\ref{th3} the leading constant 
$1$ of the regret for Lipschitz parametric classes can not be improved in general. 
We now show that for some special function $f$ one can improve the 
constant to $\frac{1}{2}$, as already noticed in
\cite{shamir2020logistic,jss20,jss21}. 
For any function $f:\mathbb{R}^d\times \mathbb{R}^d\rightarrow [0,1]$, 
we say the Hessian of $\log f$ is uniformly bounded on $\mathcal{X}\subset \mathbb{R}^d$, 
if there exists constant $C$ such that for any 
$\w\in \mathbb{R}^d$ and $\x\in \mathcal{X}$ and $y\in \{0,1\}$ we have
$$
\sup_{||\textbf{u}||_2\le 1}|\textbf{u}^T\nabla^2_{\w}\log f(\w,\x)^{y}(1-f(\w,\x))^{1-y} 
\textbf{u}|\le C,
$$
where $\nabla^2_{\w}$ is the Hessian at $\w$.
The proof of the next theorem can be found in Appendix~\ref{app:4}.

\begin{theorem}
\label{th4}
Let $f:\mathbb{R}^d\times \mathbb{R}^d\rightarrow [0,1]$ 
be a function such that the Hessian of $\log f$ is uniformly bounded by $C$ on $\mathcal{X}$. 
Let
$$
\mathcal{H}_f=\{f(\w,\x):\w\in\mathcal{W},\x\in\mathcal{X}\}
$$
be such a class of $f$ restricted to some compact set $\mathcal{W}\subset\mathbb{R}^d$. Then
\begin{equation}
\label{eq-th4}
r^a_T(\mathcal{H}_f)\le\log\frac{\text{Vol}(\mathcal{W}^*)}{\text{Vol}(\mathcal{B}_2^d(\sqrt{d/CT}))}+d/2+\log 2.
\end{equation}
where $\mathcal{W}^*=\{\w+\textbf{u}\mid \w\in \mathcal{W},~\textbf{u}\in \mathcal{B}_2^d(\sqrt{d/CT})\}$, $\text{Vol}(\cdot)$ is volume under Lebesgue measure. In particular, for $\mathcal{W}=\mathcal{B}_2^d(R)$, we have
$$r^a_T(\mathcal{H}_f)\le \frac{d}{2}\log\left(\frac{2CR^2T}{d}+2\right)+d/2+\log 2.$$
\end{theorem}

Note that, Theorem~\ref{th4} subsumes the results of~\cite{kn05, shamir2020logistic}\footnote{To get the upper bounds in~\cite{shamir2020logistic} one only needs to estimate the volume of $\ell_s$ balls, which is well known~\cite{wang2005volumes}.}, where the authors considered the function of form $f(\langle\w,\x\rangle)$ and requires that the second derivative of $\log f$ is bounded, see also~\cite[Chapter 11.10]{lugosi-book}. However, the KL-divergence-based argument of~\cite{kn05} can not be used directly in the setup of Theorem~\ref{th4} since we \emph{do not} assume the function $f$ has a linear structure. Our main proof technique of Theorem~\ref{th4} is a direct application of Lemma~\ref{lem1} and an estimation of the integrals via Taylor expansion; see Appendix~\ref{app:4} for more details on the proof.



Finally, we complete this part with the following lower bound for generalized linear functions under unit $\ell_s$ balls. See Appendix~\ref{app:5} for proof.

\begin{theorem}
\label{th5}
Let $f:\mathbb{R}\rightarrow [0,1]$ be an arbitrary function such that 
there exists $c_1,c_2\in (0,1)$ and for all $r>0$ we have  
$[c_1-c_2d^{-r},c_1+c_2d^{-r}]\subset f([-d^{-r},d^{-r}])$ for all sufficient large $d$. 
Let
$$
\mathcal{H}_f=\{f(\langle \w,\x\rangle):\w\in \mathcal{B}_s^d(1),\x\in \mathcal{B}_s^d(1)\}
$$
where $s>0$. Then
\begin{equation}
\label{eq-th5}
r^a_T(\mathcal{H}_f)\ge \frac{d}{2}\log\left(\frac{T}{d^{(s+2)/s}}\right)-O(d)
\end{equation}
where $O$ hides some absolute constant that is independent of $d,T$.
\end{theorem}

Note that for the logistic function $f(x)=(1+e^{-x})^{-1}$ Theorem~\ref{th5}
holds with $c_1=\frac{1}{2}$ and $c_2=\frac{1}{5}$. Therefore,\\ 
{\bf 1}. If $s=1$, then $$r^a_T(\mathcal{H}_f)\ge 
\frac{d}{2}\log\left(\frac{T}{d^3}\right)-O(d).$$
{\bf 2}.    If $s=2$, then $$r^a_T(\mathcal{H}_f)\ge 
\frac{d}{2}\log\left(\frac{T}{d^2}\right)-O(d).  $$
{\bf 3}.      If $s=\infty$, then $$r^a_T(\mathcal{H}_f)\ge 
\frac{d}{2}\log\left(\frac{T}{d}\right)-O(d).$$
This recovers all the lower bounds from ~\cite{shamir2020logistic}. We note that a simple sufficient condition for Theorem~\ref{th5} to hold is to require $f'(0)\not=0$ if $f(x)$ is differentiable.

\paragraph{Large Growth.}
We now present some results for large $d$ growing even faster than $T$. We will show that the size of \emph{global} sequential covering (Definition~\ref{def1}) of a class $\mathcal{H}$ can be bounded by the sequential fat-shattering number of $\mathcal{H}$ in a similar fashion as in~\cite{rakhlin2010online}. We first introduce the notion of sequential fat-shattering number as in~\cite{rakhlin2010online}.

We denote $\{0,1\}^{d}_*$ to be the set of all binary sequences of
length less than or equal to $d$. A binary tree of depth $d$ with labels in 
$\mathcal{X}$ is defined to be a map $\tau:\{0,1\}^d_*\rightarrow \mathcal{X}$. 
For any function class $\mathcal{H}\subset [0,1]^{\mathcal{X}}$, we say $\mathcal{H}$ 
$\alpha$-fat shatters tree $\tau$ if there exists 
$[0,1]$-value tree $\textbf{s}:\{0,1\}^d_*\rightarrow [0,1]$ such that for any binary
sequence $\epsilon_1^d\in \{0,1\}^d_*$ there exist $h\in \mathcal{H}$ such that for all 
$t\in [d]$:\\
1. If $\epsilon_t=0$, then
    $h(\tau(\epsilon_1^{t-1}))\le \textbf{s}(\epsilon_1^{t-1})-\alpha;$
\\
2.    If $\epsilon_t=1$, then
    $h(\tau(\epsilon_1^{t-1}))\ge \textbf{s}(\epsilon_1^{t-1})+\alpha.$

\begin{definition}
\label{def2}
The sequential $\alpha$-fat shattering number of $\mathcal{H}$ is defined to be the 
maximum number $d (\alpha)$ such that $\mathcal{H}$ $\alpha$-fat 
shatters a tree $\tau$ of depth $d:=d(\alpha)$.
\end{definition}

In the below lemma, we present an upper bound for the cardinally of the global covering set with algorithmically constructed cover set $\mathcal{G}_\alpha$
(see also ~\cite[Section 6.1]{rakhlin2010online}). Appendix~\ref{app:6} presents the  proof.


\begin{lemma}
\label{th6}
Let $\mathcal{H}$ be any class map $\mathcal{X}\rightarrow [0,1]$ 
and $d(\alpha)$ be the sequential $\alpha$-fat shattering number of $\mathcal{H}$. 
Then there exists a global sequential $\alpha$-covering set 
$\mathcal{G}_{\alpha}$ of $\mathcal{H}$ as in Definition~\ref{def1} such that
\begin{equation}
\label{eq-th6}
|\mathcal{G}_{\alpha}|\le \sum_{t=0}^{d(\alpha/3)}
\binom{T}{t}\left\lceil\frac{3}{2\alpha}\right\rceil^t\le 
\left\lceil\frac{3T}{2\alpha}\right\rceil^{d(\alpha/3)+1}.
\end{equation}
\end{lemma}

\begin{example}
\label{ex2}
{\rm 
By~\cite{rakhlin2010online} we know that the sequential $\alpha$-fat shattering 
number of linear functions $f(\w,\x)=|\langle\w,\x\rangle|$ with 
$\w,\x\in \mathcal{B}_2^d(1)$ is of order $\Tilde{O}(\alpha^{-2})$
where in $\Tilde{O}$ we hide a polylog factor.
Lemma~\ref{th6} implies that the global sequential $\alpha$-covering number is upper bounded by
$$\left\lceil \frac{(3T)}{(2\alpha)}\right\rceil^{d(\alpha/3)+1}.$$
By Theorem~\ref{th1}, we have
$$
r^a_T(\mathcal{H}_f)\le \inf_{0<\alpha<1}\left\{2\alpha T+
\Tilde{O}\left(\frac{1}{\alpha^2}\right)\right\}\le
\Tilde{O}(T^{2/3}),$$
by taking $\alpha=T^{-1/3}$. This bound is \emph{independent} of the data dimension $d$.}
\end{example}

\begin{remark}
Observe that for any class $\mathcal{H}$ with sequential fat-shattering number of order $\alpha^{-s}$ one can achieve a regret upper bound of order $\tilde{O}(T^{s/s+1})$ by Theorem~\ref{th1}. We refer to~\cite{rakhlin2010online, rakhlin2015martingale} for the estimations of sequential fat-shattering number of a variety of classes.
\end{remark}

Finally, we present the following general lower bound. See Appendix~\ref{app:7} for proof.

\begin{theorem}
\label{th7}
For any $s\ge 1$, we define
$$\mathcal{D}_s=\left\{\textbf{p}\in [0,1]^T:\sum_{t=1}^Tp_t^s\le 1\right\}.$$ 
We can view the vectors in $\mathcal{D}_s$ as functions mapping $[T]\rightarrow [0,1]$. 
Then
\begin{equation}
\label{eq-th7}
r^a_T(\mathcal{D}_s)\ge r^*_T(\mathcal{D}_s)\ge \Omega(T^{s/s+1}).
\end{equation}
\end{theorem}

To see why Theorem~\ref{th7} implies a lower bound
for $f(\bw,\bx)=|\la \bw, \bx\ra|$ with $d\ge T$, as in Example~\ref{ex2}, 
we take $\w,\x\in \mathcal{B}_2^T(1)$ (i.e., with $d=T$) 
and define $\x_t={\bf e}_t$ with ${\bf e}_t$ being the standard base of 
$\mathbb{R}^T$ that takes value $1$ at position
 $t$ and zeros otherwise. Note that the functions of 
$\mathcal{H}_f$ with $f(\w,\x)=|\langle\w,\x\rangle|$ restricted on 
$\x^T$ is exactly $\mathcal{D}_2$. Then
$$
r^a_T(\mathcal{H}_f)\ge r^*_T(\mathcal{H}_f)\ge r^*_T(\mathcal{D}_2)\ge \Omega(T^{2/3})
$$
and this is a matching lower bound of Example~\ref{ex2}. Note that, it is proved in~\cite{rakhlin2015sequential} that for function $f(\w,\x)=\frac{\langle\w,\x\rangle+1}{2}$, one can achieve the regret of form $\tilde{O}(\sqrt{T})$\footnote{A $\tilde{\Omega}(\sqrt{T})$ lower bound for $d\ge \sqrt{T}$ can be derived from Theorem~\ref{th5}, recovering~\cite[Lemma 8]{rakhlin2015sequential}.}. Example~\ref{ex2} implies that the generalized linear functions of form $f(\langle\w,\x\rangle)$ can have different regrets with polynomial gap even with a simple shift on the value (though they have the same covering number). It is therefore an interesting open problem to investigate a tighter complexity measure (instead of a covering number as in Definition~\ref{def1}) that captures this phenomenon.

\section{Conclusion}

In this paper, we presented best known lower and upper bounds on
sequential online regret
for a large class of experts.
We accomplish it by designing a new smooth truncated Bayesian algorithm,
together with a new concept of global sequential covering,
that achieves these upper bounds.
For the lower bounds, we use an novel information-theoretic approach based on the Shtarkov sum. We anticipate that these techniques could be generalized to a broader set of problems, e.g., when the covariant $\x^T$ present stochastically. We leave it to the future works.


\bibliographystyle{plain}
\bibliography{nips22}

\begin{thebibliography}{10}

\bibitem{bry98}
A.~Barron, J.~Rissanen, and B.~Yu.
\newblock The minimum description length principle in coding and modeling.
\newblock {\em IEEE Trans. Inf. Theory}, 44(6):2743--2760, Oct. 1998.

\bibitem{ben2009agnostic}
Shai Ben-David, D{\'a}vid P{\'a}l, and Shai Shalev-Shwartz.
\newblock Agnostic online learning.
\newblock In {\em COLT}, volume~3, page~1, 2009.

\bibitem{bilodeau2020tight}
Blair Bilodeau, Dylan Foster, and Daniel Roy.
\newblock Tight bounds on minimax regret under logarithmic loss via
  self-concordance.
\newblock In {\em International Conference on Machine Learning}, pages
  919--929. PMLR, 2020.

\bibitem{lugosi-book}
N.~Cesa-Bianchi and G.~Lugosi.
\newblock {\em Prediction, Learning and Games}.
\newblock Cambridge University Press, 2006.

\bibitem{cb94}
B.~Clarke and A.~Barron.
\newblock Jeffreys' prior is asymptotically least favorable under entropy of
  risk.
\newblock {\em J. Statistical PLanning and Inference}, pages 453 -- 471, 1994.

\bibitem{cs95}
I.~Csiszar and P.~Shields.
\newblock Redundancy rates for renewal and other processes.
\newblock {\em IEEE Trans. Inf. Theory}, 42:2065--2072, 1995.

\bibitem{daniely2011multiclass}
Amit Daniely, Sivan Sabato, Shai Ben-David, and Shai Shalev-Shwartz.
\newblock Multiclass learnability and the erm principle.
\newblock In {\em Proceedings of the 24th Annual Conference on Learning
  Theory}, pages 207--232. JMLR Workshop and Conference Proceedings, 2011.

\bibitem{davisson73}
L.~D. Davisson.
\newblock Universal noiseless coding.
\newblock {\em IEEE Trans. Inf. Theory}, IT-19(6):783--795, Nov. 1973.

\bibitem{ds04}
M.~Drmota and W.~Szpankowski.
\newblock Precise minimax redundancy and regrets.
\newblock {\em IEEE Trans. Inf. Theory}, IT-50:2686--2707, 2004.

\bibitem{fs02}
P.~Flajolet and W.~Szpankowski.
\newblock Analytic variations on redundancy rates of renewal processes.
\newblock {\em IEEE Trans. Information Theory}, 48:2911--2921, 2002.

\bibitem{foster18}
Dylan~J Foster, Satyen Kale, Haipeng~Luo andMehryar Mohri, and Karthik
  Sridharan.
\newblock Logistic regression: The importance of being improper.
\newblock In {\em COLT - Conference on Learning Theory}, 2018.

\bibitem{hazan14}
E.~Hazan, T.~Koren, and K.~Y. Levy.
\newblock Logistic regression: Tight bounds for stochastic and online
  optimization.
\newblock In {\em The 27th Conference on Learning Theory, COLT 2014}, pages
  197--209. MIT press, 2014.

\bibitem{jss20}
P.~Jacquet, G.~I. Shamir, and W.~Szpankowski.
\newblock Precise minimax regret for logistic regression with categorical
  feature values.
\newblock In {\em PRML: ALT'21}, volume 132, pages 755--771, 2021.

\bibitem{jss21}
Philippe Jacquet, Gil Shamir, and Wojciech Szpankowski.
\newblock Precise minimax regret for logistic regression with categorical
  feature values.
\newblock In {\em Algorithmic Learning Theory}, pages 755--771. PMLR, 2021.

\bibitem{jezequel2021mixability}
R{\'e}mi J{\'e}z{\'e}quel, Pierre Gaillard, and Alessandro Rudi.
\newblock Mixability made efficient: Fast online multiclass logistic
  regression.
\newblock {\em Advances in Neural Information Processing Systems}, 34, 2021.

\bibitem{kn05}
Sham~M Kakade and Andrew~Y. Ng.
\newblock Online bounds for bayesian algorithms.
\newblock In L.~K. Saul, Y.~Weiss, and L.~Bottou, editors, {\em Advances in
  Neural Information Processing Systems 17}, pages 641--648. MIT Press, 2005.

\bibitem{kt83}
R.~E. Krichevsky and V.~K. Trofimov.
\newblock The performance of universal encoding.
\newblock {\em IEEE Trans. Inform. Theory}, IT-27(2):199--207, Mar. 1981.

\bibitem{mcshane1934extension}
Edward~James McShane.
\newblock Extension of range of functions.
\newblock {\em Bulletin of the American Mathematical Society}, 40(12):837--842,
  1934.

\bibitem{os04}
A.~Orlitsky and N.~P. Santhanam.
\newblock Speaking of infinity.
\newblock {\em IEEE Trans. Inf. Theory}, 50(10):2215--2230, Oct. 2004.

\bibitem{rakhlin14}
A.~Rakhlin and K.~Sridharan.
\newblock Online nonparametric regression with general los function.
\newblock In {\em COLT}, 2014.

\bibitem{rakhlin2015sequential}
Alexander Rakhlin and Karthik Sridharan.
\newblock Sequential probability assignment with binary alphabets and large
  classes of experts.
\newblock {\em arXiv preprint arXiv:1501.07340}, 2015.

\bibitem{rakhlin2010online}
Alexander Rakhlin, Karthik Sridharan, and Ambuj Tewari.
\newblock Online learning: Random averages, combinatorial parameters, and
  learnability.
\newblock {\em Advances in Neural Information Processing Systems}, 23, 2010.

\bibitem{rakhlin2015martingale}
Alexander Rakhlin, Karthik Sridharan, and Ambuj Tewari.
\newblock Sequential complexities and uniform martingale laws of large numbers.
\newblock {\em Probability Theory and Related Fields}, 161(1):111--153, 2015.

\bibitem{rissanen84}
J.~Rissanen.
\newblock Universal coding, information, prediction, and estimation.
\newblock {\em IEEE Trans. Inform. Theory}, IT-30(4):629--636, Jul. 1984.

\bibitem{rissanen96}
J.~Rissanen.
\newblock Fisher information and stochastic complexity.
\newblock {\em IEEE Trans. Information Theory}, 42:40--47, 1996.

\bibitem{shamir06b}
G.~I. Shamir.
\newblock On the {MDL} principle for i.i.d.\ sources with large alphabets.
\newblock {\em IEEE Trans. Inform. Theory}, 52(5):1939--1955, May 2006.

\bibitem{ss20}
G.~I. Shamir and W.~Szpankowski.
\newblock Low complexity approximate bayesian logistic regression for sparse
  online learning.
\newblock In {\em ArXiv: http://arxiv.org/abs/2101.12113}, 2021.

\bibitem{shamir2020logistic}
Gil~I Shamir.
\newblock Logistic regression regret: What's the catch?
\newblock In {\em Conference on Learning Theory}, pages 3296--3319. PMLR, 2020.

\bibitem{shtarkov87}
Y.~M. Shtarkov.
\newblock Universal sequential coding of single messages.
\newblock {\em Problems of Information Transmission}, 23(3):3--17, Jul.-Sep.
  1987.

\bibitem{spa98}
W.~Szpankowski.
\newblock On asymptotics of certain recurrences arising in universal coding.
\newblock {\em Problems of Information Transmission}, 34:55--61, 1998.

\bibitem{spa-book}
W.~Szpankowski.
\newblock {\em Average Case Analysis of Algorithms on Sequences}.
\newblock Wiley, New York, 2001.

\bibitem{sw12}
W.~Szpankowski and M.~Weinberger.
\newblock Minimax pointwise redundancy for memoryless models over large
  alphabets.
\newblock {\em IEEE Trans. Information Theory}, 58:4094--4104, 2012.

\bibitem{vovk2001competitive}
Volodya Vovk.
\newblock Competitive on-line statistics.
\newblock {\em International Statistical Review}, 69(2):213--248, 2001.

\bibitem{wainwright2019high}
Martin~J Wainwright.
\newblock {\em High-dimensional statistics: A non-asymptotic viewpoint},
  volume~48.
\newblock Cambridge University Press, 2019.

\bibitem{wang2005volumes}
Xianfu Wang.
\newblock Volumes of generalized unit balls.
\newblock {\em Mathematics Magazine}, 78(5):390--395, 2005.

\bibitem{wu-isit22}
C.~Wus, M.~Heidari, A.~Grama, and W.~Szpankowski.
\newblock Sequential vs fixed design minimax regrets for learning.
\newblock {\em preprint}, 2022.

\bibitem{xb97}
Q.~Xie and A.~Barron.
\newblock Minimax redundancy for the class of memoryless sources.
\newblock {\em IEEE Trans. Information Theory}, pages 647--657, 1997.

\bibitem{xb00}
Q.~Xie and A.~Barron.
\newblock Asymptotic minimax regret for data compression, gambling, and
  prediction.
\newblock {\em IEEE Trans. Information Theory}, 46:431--445, 2000.

\bibitem{yamanishi1998minimax}
Kenji Yamanishi.
\newblock Minimax relative loss analysis for sequential prediction algorithms
  using parametric hypotheses.
\newblock In {\em Proceedings of the eleventh annual conference on
  Computational learning theory}, pages 32--43, 1998.

\end{thebibliography}



\newpage

\appendix

\section{Proofs of Lemma~\ref{lem1} and Lemma~\ref{lem2}}
\label{app:1}

We prove here Lemma~\ref{lem1} and Lemma~\ref{lem2}.
For the reader's convenience we repeat both lemmas.

{\bf Lemma 2}
{\it 
Let $\mathcal{G}$ be a collection of functions $g_w: \mathcal{X}^*\rightarrow [0,1], w\in \mathcal{W}$. Let $\hat{y}_t$ be the  Bayesian 
prediction rule as in Step 4 of  Algorithm~\ref{alg:1} with prior $\mu$. Then, for any $\x^T$ and $y^T$ we have
$$
\sum_{t=1}^T\ell(\hat{y}_t,y_t)\le -\log \frac{\int_{\mathcal{W}} p_w(y^T\mid \x^T)\text{d}\mu}{\int_{\mathcal{W}} 1 
\text{d}\mu},
$$
where $p_w(y^T\mid \x^T)=e^{-\sum_{t=1}^T\ell(g_w(\x^t),y_t)}$ and $\ell$ is the log-loss.
}
\begin{proof}
We first observe that for any $y\in \{0,1\}$ we have $e^{-\ell(\cdot, y)}$ is
concave over $[0,1]$. Let
$$
\lambda_{t-1}(w)=\frac{p_w(y^{t-1}\mid \x^{t-1})}{\int_{\mathcal{W}} p_w(y^{t-1}\mid \x^{t-1})\text{d}\mu}.
$$
Note that $\lambda_{t-1}(w)$ forms a probability density over $\mathcal{W}$ w.r.t. $\mu$.
By definition of $\hat{y}_t$, we have
$\hat{y}_t=\mathbb{E}_{\lambda_{t-1}}[g_w(\x^t)],$
where the expectation is over the density of $\lambda_{t-1}(w)$.
Therefore, by Jensen's inequality and the update procedure as in item 6 of Algorithm~\ref{alg:1},  we have
$$e^{-\ell(\hat{y}_t,y_t)}=e^{-\ell(\mathbb{E}[g_w(\x^t)],y_t)}\ge
\mathbb{E}[e^{-\ell(g_w(\x^t),y_t)}] =\frac{\int_{\mathcal{W}} p_w(y^t\mid \x^t)\text{d}\mu}
{\int_{\mathcal{W}} p_w(y^{t-1}\mid \x^{t-1})\text{d}\mu}.
$$
By telescoping the sum, we find
$$
e^{-\sum_{t=1}^T\ell(\hat{y}_t,y_t)}\ge \frac{\int_{\mathcal{W}} p_w(y^T\mid \x^t)\text{d}\mu}{\int_{\mathcal{W}} 1
\text{d}\mu}.
$$
This implies
$$
\sum_{t=1}^T\ell(\hat{y}_t,y_t)\le -\log\frac{\int_{\mathcal{W}} p_w(y^T\mid \x^T)\text{d}\mu}{\int_{\mathcal{W}} 1
\text{d}\mu}$$
and completes the proof.
\end{proof}

\noi
{\bf Lemma 3}
{\it
For any finite class  of experts $\mathcal{G}$
$$r^a_T(\mathcal{G})\le \log |\mathcal{G}|.$$
}
\begin{proof}
Let $\mu(w)=\frac{1}{|\mathcal{W}|}$ as in Lemma~\ref{lem1} and $\hat{y}_t$ be the Bayesian predictor with input $\mathcal{G}$ and $\mu$.  Then
\begin{align}
    \sum_{t=1}^T\ell(\hat{y}_t,y_t)&\le -\log\frac{\int_{\mathcal{W}}
p_w(y^T\mid\x^T)\text{d}\mu}{\int_{\mathcal{W}} 1 \text{d}\mu}\\
    &=-\log \int_{\mathcal{W}} p_w(y^T\mid \x^T)\text{d}\mu + \log 1\\
    &=-\log \int_{\mathcal{W}} p_w(y^T\mid \x^T)\text{d}\mu\\
    &\le -\log p_{w^*}(y^T\mid \x^T)+\log|\mathcal{W}|,~\text{where }
w^*\text{ maximizes }p_w(y^T\mid \x^T)\\
    &=\sum_{t=1}^T\ell(g_{w^*}(\x^t),y_t)+\log|\mathcal{G}|,~
\text{ since }|\mathcal{W}|=|\mathcal{G} |.
\end{align}
This concludes the proof.
\end{proof}

\section{Proof of Lemma~\ref{lem3}}
\label{app:2}

We construct the set $\tilde{\mathcal{G}}$ as in Algorithm~\ref{alg:2}. For any 
$g\in \mathcal{G}$ we define a smooth truncated function
$\tilde{g}$ such that for any $\x^t\in \mathcal{X}^*$
$$\tilde{g}(\x^t)=\frac{g(\x^t)+\alpha}{1+2\alpha}.$$
We introduce the following 
short-hand notation, for any function $f$ we define
$$f(y_t)=f(\x^t)^{y_t}(1-f(\x^t))^{1-y_t}.$$

For any $\x^T,y^T$ and $h\in \mathcal{H}$, let $g\in \mathcal{G}$ 
be a $\alpha$-covering of $h$ and $\tilde{g}$ be the truncated function as 
defined above. For any $t$, we consider two cases.

\textbf{Case 1}: If $y_t=1$, we have:
\begin{align}
    \frac{h(y_t)}{\tilde{g}(y_t)}&=\frac{h(\x^t)}{\tilde{g}(\x^t)},~\text{since } y_t=1\\
    &\le \frac{g(\x^t)+\alpha}{\tilde{g}(\x^t)},~g\text{ is }\alpha\text{ cover of }h\\
    &= \frac{g(\x^t)+\alpha}{(g(\x^t)+\alpha)/(1+2\alpha)},~\text{ definition of }\tilde{g}\\
    &=1+2\alpha
\end{align}

\textbf{Case 2}: If $y_t=0$, we have
\begin{align}
    \frac{h(y_t)}{\tilde{g}(y_t)}&=\frac{1-h(\x^t)}{1-\tilde{g}(\x^t)}\\
    &\le \frac{1-g(\x^t)+\alpha}{1-\tilde{g}(\x^t)},~g\text{ is }\alpha\text{ cover of }h\\
&=\frac{1-g(\x^t)+\alpha}{1-(g(\x^t)+\alpha)/(1+2\alpha)}~\text{, definition of }\tilde{g}\\
    &=\frac{1-g(\x^t)+\alpha}{(1-g(\x^t)+\alpha)/(1+2\alpha)}\\
    &= 1+2\alpha,
\end{align}

Now,  combining the two cases, we have
\begin{align}
\frac{p_h(y^T\mid \x^T)}{p_{\tilde{g}}(y^T\mid \x^T)}&=\prod_{t=1}^T\frac{h(y_t)}{\tilde{g}(y_t)}\\
&\le \left(1+2\alpha\right)^T.
\end{align}
This completes the proof of Lemma~\ref{lem3}.

\section{Proof of Theorem~\ref{th3}}
\label{app:3}

We need the following two lemmas.

\begin{lemma}
\label{lem4}
Let $\mathcal{P}$ be a finite class of distributions over the same domain $\mathcal{X}$. Denote
$$S=\sum_{x\in \mathcal{X}}\max_{p\in \mathcal{P}}p(x)$$
be the Shtarkov sum. Then for any estimation rule 
$\Phi:\mathcal{X}\rightarrow \mathcal{P}$ we have
$$\max_{p\in \mathcal{P}}p(\Phi(x)\not=p)\ge 1-\frac{S}{|\mathcal{P}|}.$$
\end{lemma}
\begin{proof}
Note that $\Phi$ partitions $\mathcal{X}$ into $|\mathcal{P}|$ disjoint parts. 
For any $p\in \mathcal{P}$, we denote $\mathcal{X}_p=\{x\in \mathcal{X}:\Phi(x)=p\}$ 
be the partition corresponding to $p$.  We have
$$\sum_{p\in \mathcal{P}}p(\mathcal{X}_p)= \sum_{x\in \mathcal{X}}p_x(x)\le 
\sum_{x\in \mathcal{X}}\max_{p\in \mathcal{P}}p(x)=S,$$
where $p_x\in \mathcal{P}$ is the distribution such that $x\in \mathcal{X}_{p_{x}}$. 
This implies $$\min_{p\in \mathcal{P}}p(\mathcal{X}_p)\le\frac{S}{|\mathcal{P}|}.$$
The result follows by taking the complements of $\mathcal{X}_p$.
\end{proof}

\begin{lemma}
\label{lem5}
For any $M$ and $T\gg \log M$, there exist $M$ vectors 
$v_1,v_2,\cdots,v_M\in \{0,1\}^T$ such that for any $i\not=j\in [M]$ we have
$$\sum_{t=1}^T1\{v_i[t]\not=v_j[t]\}\ge T/4.$$
\end{lemma}

\begin{proof}
We use a probabilistic argument to construct the vectors $v_i$. 
To do so, we choose each $v_i$ independently from uniform distribution 
over $\{0,1\}^T$. Now, it easy to see that the expected distinct coordinates 
of any two vectors is $T/2$. The result follows by a simple application of 
Chernoff bound  (e.g., see \cite{spa-book})
by showing that the event of the lemma has a positive probability.
\end{proof}

Now we are in the position to prove Theorem~\ref{th3}.
Let $\x_1,\cdots,\x_T\in \mathbb{R}^d$ be any distinct points. 
We will construct a $L$-Lipschitz function $f(\w,\x)$ such that the regret 
restricted only on $\x^T$ is large. To do so, we consider a maximum packing 
$M$ of the parameter space $\mathcal{B}_2^d(R)$ of radius 
$\alpha/L>0$ (where $\alpha$ is to be determined latter). 
Standard volume argument (see Chapter 5 of~\cite{wainwright2019high}) yields that
$$|M|\ge \left(\frac{LR}{2\alpha}\right)^d.$$

Now, we will define a $L$-Lipschitz functions $f(\w,\x)$ only on 
$\w\in M$ and $\x\in \{\x_1,\cdots,\x_T\}$. By Lemma~\ref{lem5} (assume for now the 
conditions are satisfied), we can find  $|M|$ binary vectors $V\subset \{0,1\}^T$ 
such that any pair of the vectors has Hamming distance lower bounded by $T/4$. For
 each of the vector $v\in V$, we define a vector $u\in [0,1]^T$ in 
the following way, for all $t\in [ T]$
\begin{itemize}
    \item[1.] If $v[t]=0$ then set $u[t]=0$;
    \item[2.] If $v[t]=1$ then set $u[t]=\alpha$.
\end{itemize}

Denote by  $U$ be the set of all such vectors $u$. Note that $|U|=M$.
For any $\w\in M$, we can associate a unique 
$u\in U$ such that for all $t\in [T]$ 
$$f(\w,\x_t)=u[t].$$

We now show that $f$ is indeed $L$-Lipschitz restricted on $M$ 
for all $\x_t\in\{\x_1,\cdots,\x_T\}$.  This is because for any 
$\w_1\not=\w_2\in M$ we have $|f(\w_1,\x_t)-f(\w_2,\x_t)|\le \alpha$ by definition 
of $U$ and $||\w_1-\w_2||_2\ge \alpha/L$ since $M$ is a packing.

We now view the vectors in $u\in U$ as a product of Bernoulli distributions with 
each coordinate $t$ independently sampled from $\mathrm{Bern}(u[t])$. 
We show that the sources in $U$ are identifiable. To see this, we note that for 
any distinct pairs $u_1,u_2\in U$, there exist a set $I\in [T]$ such that $u_1$ 
and $u_2$ differ on $I$ and $|I|\ge T/4$. This further implies that 
there exist a set $J\subset I$ with $|J|\ge T/8$ such that $u_1$ takes all 
$0$ on $J$ and $u_2$ takes all $\alpha$ on $J$ (or vice versa). We can then 
distinguish $u_1,u_2$ by checking if the samples on $J$ are all $0$s or not. The 
probability of making error is upper bounded by
$$(1-\alpha)^{T/8}\le e^{-\alpha T/8}.$$
Since there are only $|M|^2$ such pairs, we have the probability 
of wrongly identifying the source  upper bounded by
$$|M|^2e^{-\alpha T/8}.$$
Taking $\alpha = \frac{16d\log (RLT)}{T}$, the error probability is upper bounded by
$$\left(\frac{RLT}{32d\log (RLT)}\right)^{2d}e^{-2d\log(RLT)}\le 
\left(\frac{1}{32d\log(RLT)}\right)^{2d}\le \frac{1}{2},$$
for sufficient large $d,T$, where we have use the fact that 
$|M|\le (\frac{RLT} {32d\log(RLT)})^d$. Note that we only showed a 
lower bound on $|M|$ before, but this is not a problem
 since we can always remove some points from $M$ to make the upper bound holds as well.

By Lemma~\ref{lem4}, we know that the Shtarkov sum of sources in 
$U$ is lower bounded by $|M|/2$. Therefore, we have
$$r^a_T(\mathcal{H}_f)\ge r^*_T(\mathcal{H}_f)\ge \log(|M|/2)\ge d\log 
\left(RLT/d\right) - d\log 64-
d\log\log (RLT).
$$

Now, we have to extend the function to the whole set $\mathcal{B}_2^d(R)$ 
and keep the $L$-Lipschitz property. This follows from a classical result 
in real analysis (see~\cite[Theorem 1]{mcshane1934extension}) by 
defining for all $\w\in \mathcal{B}_2^d(R)$ and $\x_t\in \{\x_1,\cdots,\x_T\}$
$$f(\w,\x_t)=\sup_{\w'\in M}\{f(\w',\x_t)-L||\w-\w'||_2\}.$$
For the $\x\not\in \{\x_1,\cdots,\x_T\}$, we can simply let $f(\w,\x)=0$ for all $\w$.

Finally, we need to check that the condition of Lemma~\ref{lem5} holds for our 
choice of $\alpha$, this is satisfied by our assumption $T\gg d\log (RLT)$.

\section{Proof of Theorem~\ref{th4}}
\label{app:4}

To make the proof more transparent, we only prove the case for $\mathcal{W}=\mathcal{B}_2^d(R)$ since the proof for other compact $\mathcal{W}$ follows similar path. Note that, for $\mathcal{W}=\mathcal{B}_2^d(R)$, we have $\mathcal{W}^*=\mathcal{B}_2^d(R+\sqrt{d/CT})$.

The proof resembles that of~\cite{foster18} but running the Bayesian predictor (Algorithm~\ref{alg:1}) over 
$\mathcal{W}^*$ instead of $\mathcal{W}$ with $\mathcal{G}$ being $\mathcal{H}_f$ and $\mu$ being Lebesgue measure. 
Let $\x^T$, $y^T$ and $\hat{y}^T$ be the feature, label and 
predictions of the Bayesian predictor respectively. By Lemma~\ref{lem1}
\begin{align}
    \sum_{t=1}^T\ell(\hat{y}_t,y_t)\le -\log \frac{\int_{\mathcal{B}_2^d(R+\sqrt{d/CT})}
p_{\w}(y^T\mid \x^T)\text{d}\mu}{\int_{\mathcal{B}_2^d(R+\sqrt{d/CT})}1\text{d}\mu},
\end{align}
where $\mu$ is the Lebesgue measure and
$$p_{\w}(y^T\mid \x^T)=\prod_{t=1}^Tf(\w,\x_t)^{y_t}(1-f(\w,\x_t))^{1-y_t}.$$

We now write $h_t(\w)\overset{\text{def}}{=}\log f(\w,\x_t)^{y_t}(1-f(\w,\x_t))^{1-y_t}$ to simplify notation. 
It is easy to see that $\ell(f(\w,\x_t),y_t)=-h_t(\w)$.

Let $\w^*$ be the point in $\mathcal{B}_2^d(R)$ that maximizes
$$h(\w)\overset{\text{def}}{=}\sum_{t=1}^Th_t(\w).$$
Let $\textbf{u}=\nabla h(\w^*)$ be the gradient of $h$ at $\w^*$. 
By Taylor theorem, we have for any $\w\in \mathcal{B}_2^d(R+\sqrt{d/CT})$
$$h(\w)=h(\w^*)+\textbf{u}^T(\w-\w^*)+\frac{1}{2}(\w-\w^*)^\tau\nabla^2_{\w'}h(\w')(\w-\w^*),$$
where $\w'$ is a convex combination of $\w$ and $\w^*$ and $\textbf{u}^\tau$ is the transpose of $\textbf{u}$.

Now, the key observation is that for any point $\w$ such that $\textbf{u}^\tau(\w-\w^*)\ge 0$ we have
\begin{equation}
\label{thm4eq1}
    h(\w)\ge h(\w^*)+\frac{1}{2}(\w-\w^*)^\tau\nabla^2_{\w'}h(\w')(\w-\w^*)\ge h(\w^*)-\frac{1}{2}C T ||\w-\w^*||_2^2,
\end{equation}
where the last inequality follows from our assumption about the bounded Hessian of $\log f$. 
Let $B$ be the half ball of radius $\sqrt{d/CT}$ centered at $\w^*$ such 
that for all $\w\in B$ we have $\textbf{u}^T(\w-\w^*) \ge 0$. 
By~(\ref{thm4eq1}), for all $\w\in B$
\begin{equation}
\label{thm4eq2}
    h(\w)\ge h(\w^*)-\frac{1}{2}CT(\sqrt{d/CT})^2=h(\w^*)-d/2.
\end{equation}

Note that $B\subset \mathcal{B}_2^d(R+\sqrt{d/CT})$. Then using above observations we arrive at
\begin{align}
        \sum_{t=1}^T\ell(\hat{y}_t,y_t)&\le -\log \frac{\int_{\mathcal{B}_2^d(R+\sqrt{d/CT})}
p_{\w}(y^T\mid \x^T)\text{d}\mu}{\int_{\mathcal{B}_2^d(R+\sqrt{d/CT})}1\text{d}\mu}\\
        &\le-\log\frac{\int_{B}p_{\w}(y^T\mid \x^T)\text{d}\mu}{\int_{\mathcal{B}_2^d(R+
\sqrt{d/CT})}1\text{d}\mu}\\
        &\le -\log\frac{e^{-d/2}\int_{B}p_{\w^*}(y^T\mid \x^T)\text{d}\mu}{\int_{\mathcal{B}_2^d
(R+\sqrt{d/CT})}1\text{d}\mu}\\
        &=-\log p_{\w^*}(y^T\mid \x^T) + d/2- \log\frac{\text{Vol}(B)}{\text{Vol}
(\mathcal{B}_2^d(R+\sqrt{d/CT}))}\\
        &=-\log p_{\w^*}(y^T\mid \x^T) + d/2-\log \frac{\frac{1}{2}\sqrt{\frac{d}{CT}}^d}
{(R+\sqrt{d/CT})^d}\\
        &\le -\log p_{\w^*}(y^T\mid \x^T) + d/2+\frac{d}{2}\log\left(\frac{2CR^2T}{d}+2\right)+\log 2\\
        &=\sum_{t=1}^{T}\ell(f(\w^*,\x_t),y_t) + 
\frac{d}{2}\log\left(\frac{2CR^2T}{d}+2\right)+d/2+\log 2.
\end{align}
This completes the proof of Theorem~\ref{th4}.
\begin{remark}
When compared to the technique in~\cite{yamanishi1998minimax}, Theorem~\ref{th4} does not assume that the gradient critical point of the loss is zero (e.g., the minimum  may occur on the boundary). This is why we need to restrict to the half ball $B$ in order to discard the linear term of Taylor expansion in Equation~(\ref{thm4eq2}). Moreover, in the proof we work directly on the continuous space instead of a discretized cover, giving an efficient algorithm provided the posterior is efficiently samplable (by e.g., assuming some log-concavity of $f$ as in~\cite{foster18}).
\end{remark}


\section{Proof of Theorem~\ref{th5}}
\label{app:5}

We start with the following technical lemma.

\newcommand{\bfy}{\textbf{y}}

\begin{lemma}
\label{lem6}
The following inequality holds, for $r>0$:
\begin{equation}\label{eq:Pf}
   \sum_{\bfy\in \{0,1\}^{T/d}}\sup_{w\in[c_1-c_2d^{-r},c_1+c_2d^{-r}]}P(\textbf{y}\mid w)
\ge \Omega(\sqrt{T/d^{2r+1}}),
\end{equation}
where $P(\textbf{y}\mid w)=w^{k}(1-w)^{T/d-k}$ with $k$ being the number of $1$s in $\textbf{y}$.
\end{lemma}
\begin{proof}
By Stirling approximation, for all $k\in [T/d]$, there exists a 
constant $C\in \mathbb{R}^+$ such that
\begin{align*}
B(k,T/d)&\overset{\text{def}}{=}\binom{T/d}{k}\left(\frac{k}{T/d}\right)^k
\left(1-\frac{k}{T/d}\right)^{T/d-k}\\
&\ge C\sqrt{\frac{T/d}{k(T/d-k)}}.
\end{align*}
Since $P(\textbf{y}\mid w)$ achieves maximum at $w=k*d/T$, we have
$$ \sum_{\bfy\in \{0,1\}^{T/d}}\sup_{w\in[c_1-c_2d^{-r},c_1+c_2d^{-r}]}p(\textbf{y}\mid w)\ge
\sum_{k=c_1T/d-c_2T/d^{r+1}}^{c_1T/d+c_2T/d^{r+1}}B(k,T/d).
$$
Therefore, for each $k$ in the above summation, we have that
$$
\frac{1}{\sqrt{k(T/d-k)}}\ge \sqrt{(c_1+c_2d^{-r})(1-c_1-c_2d^{-r})}d/T.
$$
Therefore, the LHS of \eqref{eq:Pf} is lower bounded by
$$
C\sqrt{(c_1+c_2d^{-r})(1-c_1-c_2d^{-r})}\sqrt{\frac{T}{d}}\frac{2c_2}{d^{r}}=
\Omega(\sqrt{T/d^{2r+1}})
$$
for sufficient large $d$.
\end{proof}

Now we are ready to prove Theorem~\ref{th5}.
We choose a particular $\x^T$: 
We split the $\x^T$ into $d$ blocks each with length of $T/d$. With that, 
the $i$th part of the inputs and the outputs are denoted by
$\x^{(i)}=(\x_{(T/d)*(i-1)+1},\cdots,\x_{(T/d)*i})$ and $\bfy^{(i)}
=(y_{(T/d)*(i-1)+1},\cdots,y_{(T/d)*i})$, respectively. 
We define for any $\x_t$ in the $i$th block $\x^{(i)}$ equals ${\bf e}_i$ the standard
$d$ base of $\mathbb{R}^d$ with $1$ in position $i$ and $0$s otherwise. Note that, with these choice of $\x_t$s, we have $\langle \w,\x_t\rangle = w_i$, where $w_i$ is the $i$th coordinate of $\w$ and $\x_t\in \x^{(i)}$.

We will lower bound $r^*_T(\mathcal{H}_f\mid \x^T)$, which will automatically 
give a lower bound on $r^a_T(\mathcal{H}_f)$. We only need to compute the 
following Shtarkov sum
\begin{equation}
S_T(\mathcal{H}_{f} | \x^T)=\sum_{y^T\in \{0,1\}^T}\sup_{\w\in \mathcal{B}_{s}^{d}(1)}\prod_{i=1}^{d}
P_f(\bfy^{(i)}|w_i),
\end{equation}
where $P_f(\bfy^{(i)}|w_i)=f(w_i)^{k_i}(1-f(w_i))^{T/d-k_i}$ with $k_i$ being the number of
$1$s in $\bfy^{(i)}$. We observe 
\begin{align*}
  S_T(\mathcal{H}_{f} | \x^T) &\ge
\sum_{y^T\in \{0,1\}^T}\prod_{i=1}^{d}\sup_{w_i\in [-d^{-1/s},d^{-1/s}]}P_f(\bfy^{(i)}|w_i) \\
&=\prod_{i=1}^{d}\sum_{\bfy^{(i)}\in \{0,1\}^{T/d}}
\sup_{w_i\in [-d^{-1/s},d^{-1/s}]}P_f(\bfy^{(i)}|w_i)\\
&= \Big(\sum_{\bfy\in \{0,1\}^{T/d}}\sup_{w\in [-d^{-1/s},d^{-1/s}]}P_f(\bfy|w)\Big)^d\\
&\ge \left(\sum_{\bfy\in \{T/d\}}\sup_{w\in [c_1-c_2d^{-1/s},c_1+
c_2d^{-1/s}]}P(\bfy\mid w)\right)^d
\end{align*}
where $P(\bfy\mid w)$ is as in Lemma~\ref{lem6} and the last inequality 
holds since $[c_1-c_2d^{-1/s},c_1+c_2d^{-1/s}]\subset f([d^{-1/s},d^{-1/s}])$ 
by the assumption. Now, Lemma~\ref{lem6} implies that
$$S_T(\mathcal{H}_{f} \mid \x^T)\ge c^d\left(\frac{T}{d^{(s+2)/s}}\right)^{d/2},$$
where $c$ is some absolute constant that is independent of $d,T$. We conclude
$$
r^a_T(\mathcal{H}_f)\ge r^*_T(\mathcal{H}_f)\ge \log S_T(\mathcal{H}_{f} | \x^T)\ge 
\frac{d}{2}\log \left(\frac{T}{d^{(s+2)/s}}\right)-O(d)
$$
which completes the proof.

\section{Proof of Lemma~\ref{th6}}
\label{app:6}

We first introduce a discretized notion of fat-shattering number, 
which can be viewed as a misspecified Littlestone dimension~\cite{ben2009agnostic, daniely2011multiclass}, see also~\cite{rakhlin2015martingale}. 
For any $\alpha>0$, we can choose $K\le \lceil 1/2\alpha\rceil$ points 
$z_1<z_2\cdots<z_K$ in the interval $[0,1]$ such that any point in $[0,1]$ is 
$\alpha$ close to some $z_k$ and $z_{k+1}-z_k= 2\alpha$ for all $k\in[K]$. 
Now, we define a discretized class $\mathcal{H}'$ for the $[0,1]$-valued class 
$\mathcal{H}$ in the following way. For any $h\in \mathcal{H}$, we define function 
$h'\in \mathcal{H}'$ such that for any $\x\in \mathcal{X}$ we have
$$h'(\x)=\arg\min_{z_k\in \{z_1,\cdots,z_K\}}|z_k-h(\x)|,$$
where we break ties arbitrarily.

We now view the functions in $\mathcal{H}'$ as functions map 
$\mathcal{X}\rightarrow [K]$ (i.e., we view each $z_k$ as its index $k$). 
For any discretized class $\mathcal{H}'$, we define the discretized
$1$-shattering as follows. For any $\mathcal{X}$-valued tree $\tau$ of depth $d$, 
we say $\mathcal{H}'$ $1$-shatters $\tau$, if there exists $[K]$-valued tree 
$\text{s}:\{0,1\}^d_*\rightarrow [K]$ such that for any $\epsilon_1^d\in \{0,1\}^d_*$ 
there exist $h'\in \mathcal{H}'$ such that for all $t\in[d]$:
\begin{itemize}
    \item[1.] If $\epsilon_t=0$, then $h'(\tau(\epsilon_1^{t-1}))\le \textbf{s}(\epsilon_1^{t-1})-1.$
    \item[2.]if $\epsilon_t=1$, then
    $h'(\tau(\epsilon_1^{t-1}))\ge \textbf{s}(\epsilon_1^{t-1})+1.$
\end{itemize}

\begin{definition}
\label{def3}
The discretized $1$-shattering number of a discretized class $\mathcal{H}'$ is 
defined to be the maximum number $d$ such that $\mathcal{H}'$ 
$1$-shatters some tree $\tau$ of depth $d$. This number is denoted as $\textbf{FAT}_1(\mathcal{H}')$. If no such tree exists, we define the $1$-shattering number to be $0$ if $\mathcal{H}'$ is non-empty and $-1$ if $\mathcal{H}'$ is empty.
\end{definition}

The proof of Lemma~\ref{th6} follows from the following three lemmas.

\begin{lemma}
\label{lem7}
The discretized $1$-shattering number of $\mathcal{H}'$ 
is upper bounded by the $\alpha$-fat shattering number of 
$\mathcal{H}$ where $\mathcal{H}'$ is the discretized class of $\mathcal{H}$ at scale $\alpha$.
\end{lemma}
\begin{proof}
Let $\tau$ be the tree of depth $d$ that is shattered by 
$\mathcal{H}$ with a $[K]$-valued tee $\textbf{s}$. We define a 
$[0,1]$-valued tree $\textbf{s}'$ as follows for any $\epsilon_1^t\in \{0,1\}^d_*$, 
$$\textbf{s}'(\epsilon_1^{t-1})=z_{\textbf{s}(\epsilon_1^{t-1})}.
$$
We now show that the $\tau$ and $\textbf{s}'$ are  the desired pair that is 
$\alpha$-shattered by $\mathcal{H}$. This follows from the fact that for any $z_k$ 
and $z_{l}$ with $k\not=l$ if some $y\in [0,1]$ is closer to $z_l$ , then
$$|y-z_{k}|\ge \alpha$$
as easy to see.
\end{proof}

For any discretized class $\mathcal{H}'$, we say a class 
$\mathcal{G}$ of functions map $\mathcal{X}^*\rightarrow [K]$ $1$-covers 
$\mathcal{H}'$ if for any $\x_1,\cdots,\x_T\in \mathcal{X}$ and $h'\in \mathcal{H}'$ 
there exists $g\in \mathcal{G}$ such that for all $t\in [T]$ 
$$|h'(\x_t)-g(\x^t)|\le 1.$$
The following result is crucial for our following analysis, 
which is an analogy of Lemma 12 of~\cite{ben2009agnostic} (see also~\cite{daniely2011multiclass, rakhlin2010online}).

\begin{algorithm}[t]
\caption{M-SOA algorithm}\label{alg:app}
\textbf{Input}: Hypothesis class $\mathcal{H}$ with functions map $\mathcal{X}\rightarrow [K]$

\begin{algorithmic}[1]
\State Let $\mathcal{H}^*=\mathcal{H}$
\For{$t=1,\cdots,T$}
\State Receive feature $\x_t$
\State For $k\in [K]$, let 
        $$\mathcal{H}^*_{(\x_t,k)}\overset{\text{def}}{=}\{h\in \mathcal{H}^*\mid h(\x_t)=k\}$$
\State Make prediction
        $$\hat{y}_t=\arg\max_{k\in [K]}\textbf{FAT}_1(\mathcal{H}^*_{(\x_t,k)})$$
     \hspace{0.16in}   (where we break ties arbitrarily and deal with empty classes as in Definition~\ref{def3})
\State Receive label $y_t$
\State If $|\hat{y}_t-y_t|\ge 2$, set
        $$\mathcal{H}^*=\mathcal{H}^*_{(\x_t,y_t)}$$
\State If $|\hat{y}_t-y_t|< 2$, set
        $$\mathcal{H}^*=\mathcal{H}^*$$
\EndFor
\end{algorithmic}

\end{algorithm}

\begin{lemma}\label{lem8}
Suppose the discretized $1$-shattering number of $\mathcal{H}'$ is upper bounded by 
$d$, then there exists a $1$-covering set $\mathcal{G}$ of $\mathcal{H}'$ such that
$$|\mathcal{G}|\le \sum_{t=0}^d\binom{T}{t}K^t\le (TK)^{d+1}.$$
\end{lemma}
\begin{proof}
We now describe an algorithm that is similar to the SOA algorithm of~\cite{ben2009agnostic}, 
which we will call it M-SOA (Algorithm~\ref{alg:app})\footnote{The major difference with the standard SOA is steps 7-8 and where "M" stands for misspecified.}. The algorithm goes as follows: it maintains a running hypothesis class $\mathcal{H}^*$, initially equals $\mathcal{H}'$. Let $(\x_t,y_t)$ be the sample label pair received at round $t$. We will denote by $\mathcal{H}^*_{(\x_t,y_t)}$ the functions in $\mathcal{H}^*$ that is consistent with $(\x_t,y_t)$, i.e., for all $h\in \mathcal{H}^*_{(\x_t,y_t)}$ we have
$$h(\x_t)=y_t.$$

At time step $t$, the algorithm M-SOA will predict $k\in [K]$ such that 
$\textbf{FAT}_1(\mathcal{H}^*_{(\x_t,k)})$ is maximum, where we denote by $\textbf{FAT}_1(\mathcal{H}^*_{(\x_t,k)})$ the discretised 1-shattering number of $\mathcal{H}^*_{(\x_t,k)}$ and break ties arbitrarily. After receiving the true label $y_t$, the M-SOA algorithm will do the following. If $|\hat{y}_t-y_t|\ge 2$, then it sets $\mathcal{H}^*=\mathcal{H}^*_{(\x_t,y_t)}$. Else, it remains on the same $\mathcal{H}^*$. We then continue the prediction procedure for the next time step with the new $\mathcal{H}^*$.

We say the algorithm M-SOA makes an error at time step $t$ if $|\hat{y}_t-y_t|\ge 2$ 
where $\hat{y}_t$ is the prediction given by M-SOA at time step $t$. 
We claim that the M-SOA will make at most $d$ errors if the
samples $(\x^T,y^T)$ is consistent with some $h\in \mathcal{H}'$.

To see this, we prove by induction on $d$ and $T$ (the base case for $d=0$ or $T=0$ is 
easy to check). Suppose we have observed $\x_1$ at the first step. 
We show that there can not be two element $k_1,k_2\in [K]$ such that $|k_1-k_2|\ge 2$ 
and both $\mathcal{H}'_{(\x_1,k_1)}$ and $\mathcal{H}'_{(\x_1,k_2)}$ has discretized 
$1$-shattering number $\ge d$. Otherwise, we can concatenate the shattering tree of 
$\mathcal{H}'_{(\x_1,k_1)}$ and $\mathcal{H}'_{(\x_1,k_2)}$ with the root labeled by $\x_1$ to form a depth
 $d+1$ shattering tree of $\mathcal{H}'$ (with $\textbf{s}(\phi)$ being any number $\in(k_1,k_2)$). This is a contradiction, 
since the discretized $1$-shattering number of $\mathcal{H}'$ is upper bounded by $d$. 
This shows that either we will make no error at the first step or the discretized 
$1$-shattering number decreased by at least $1$ on the remaining consistent class of functions 
(after $y_1$ has been revealed). For the first case, by induction hypothesis for
$T-1$ we have the number of errors is at most $d$. For the second case, we 
also have the number of errors upper bounded by $d-1+1=d$.

We now follow the idea from the proof of Lemma 12 of~\cite{ben2009agnostic} to 
construct a covering set $\mathcal{G}$. For any subset $I\subset [T]$ of size $|I|\le d$ 
and $\{k_t\}_{t\in I}\in [K]^{|I|}$, we define a function $g$ by running our 
M-SOA algorithm by changing steps $7-8$ as follows. At each time step $t\in [I]$, we update $\mathcal{H}^*=\mathcal{H}^*_{(\x_t,k_t)}$. Otherwise, 
for any $t\not\in I$, we remain on the same $\mathcal{H}^*$. The values of $g$ 
for each $\x^t$ is given by the output of M-SOA at time step $t$.

Since the M-SOA will make at most $d$ errors if the sample-label pairs 
$(\x^T,y^T)$ are consistent with some function in $\mathcal{H}'$, we 
know that any $h\in \mathcal{H}'$ is $1$-covered by the function generated by running M-SOA 
with some $I$ and $\{k_t\}_{t\in I}$ in the above fashion.
To complete we observe that by a
simple counting argument the number of such pairs 
$I$ and $\{k_t\}_{t\in I}$ is at most 
$$\sum_{t=0}^d\binom{T}{t}K^t$$
which completes the proof.
\end{proof}

Finally, we need the following lemma that relates $1$-covering of 
$\mathcal{H}'$ with global sequential $\alpha$-covering of $\mathcal{H}$.

\begin{lemma}
\label{lem9}
Suppose there exist a $1$-covering set $\mathcal{G}$ of $\mathcal{H}'$, 
then there exists a global $3 \alpha$-covering $\mathcal{G}'$ of $\mathcal{H}$ 
such that $|\mathcal{G}|=|\mathcal{G}'|$, where $\mathcal{H}'$ is the discretised class of $\mathcal{H}$ at scale $\alpha$.
\end{lemma}
\begin{proof}
For any $g\in \mathcal{G}$, we define a function $g'$ such that for all $\x^t$ we have
$$g'(\x^t)=z_{g(\x^t)}.$$
The claim follows from the fact that any $y$ that is closest to $z_k$ 
satisfies $|y-z_k|\le \alpha$ and if some $z$ $1$-covers $z_k$ 
then we have $|z-z_k|\le 2\alpha$, by triangle inequality
$$|y-z|\le 3\alpha$$
as needed.
\end{proof}

The proof of Lemma~\ref{th6} follows from Lemma~\ref{lem7}, Lemma~\ref{lem8},
and Lemma~\ref{lem9}.

\section{Proof of Theorem~\ref{th7}}
\label{app:7}

It is sufficient to compute the Shtarkov sum as in (\ref{eq-shtarkov}).
For any $y^T\in \{0,1\}^T$ with $k$ $1$s, we claim that
$$\sup_{\textbf{p}\in\mathcal{D}_s}p(y^T)=\frac{1}{k^{k/s}},$$
where
$$
p(y^T)=\prod_{t=1}^Tp_t^{y_t}(1-p_t)^{1-y_t}.
$$
To see this, we use a \emph{perturbation} argument. Denote $I$ be the positions
in $y^T$ that takes value $1$ such that $|I|=k$. For any $\textbf{p}$
such that $p(y^T)$ is maximum, we must have $p_j=0$
for all $j\not\in I$. Suppose otherwise, we then can move some probability mass on
$p_j$ to some $p_i<1$ with $i\in I$, which will increase the value of
$p(y^T)$, thus a contradiction. Now, we need to show that
$$\prod_{i\in I}p_i\le \frac{1}{k^{k/s}},$$
this follows easily by AM-GM (i.e., arithmetic mean vs geometric mean) inequality since
$\sum_{i\in I}p_i^s\le 1$ and it takes equality when
$p_i=\frac{1}{k^{1/s}}$ for all $i\in I$. Now, the Shtarkov sum can be written as
\begin{equation}
    \sum_{k=0}^T\binom{T}{k}\frac{1}{k^{k/s}}.
\end{equation}
To find a lower bound, we only need to estimate the maximum term in the summation. We have
$$\max_{k}\binom{T}{k}\frac{1}{k^{k/s}}\ge \max_{k}\frac{T^k}{k^{(1+1/s)k}}\ge e^{\frac{s+1}{s\cdot e}T^{s/s+1}}
,$$
where the last inequality follows by taking $k=\frac{1}{e} T^{s/s+1}$, and we also use the fact that
$$\binom{T}{k}\ge \frac{T^k}{k^k}.$$
Therefore, we have
$$
r^*_T(\mathcal{D}_s)\ge \frac{s+1}{s\cdot e}T^{s/s+1}=\Omega(T^{s/s+1})
$$
which completes the proof.

\end{document}